\documentclass[journal]{IEEEtran}

\ifCLASSINFOpdf
\else
   \usepackage[dvips]{graphicx}
\fi
\usepackage{url}
\usepackage{graphicx}
\usepackage{amsmath,amssymb,amsfonts}
\usepackage{cite}
\usepackage{algorithmic}
\usepackage{algorithm}
\usepackage{xcolor}
\usepackage{array}
\usepackage{booktabs}
\usepackage{adjustbox}
\usepackage[backref]{hyperref}
\newtheorem{proposition}{Proposition}

\begin{document}
\bibliographystyle{ieeetr}
\IEEEtriggeratref{14}

\title{Progressive Coupling Method for Diffusion Posterior Sampling in Ptychography}

\author{Yihan Qiu, Mingtao Huang, Tianyuan Li, Shudong Li, and Yuan Shen%
\thanks{The authors are with the Department of Electronic Engineering, Beijing National Research Center for Information Science and Technology (BNRist), Tsinghua University, Beijing 100084, China (e-mail: qyh23@mails.tsinghua.edu.cn; huangmt@mail.tsinghua.edu.cn; li-ty22@mails.tsinghua.edu.cn; sli248@jh.edu; 
shenyuan\_ee@tsinghua.edu.cn).}
}

\maketitle

\begin{abstract}
Ptychography is a computational imaging method that reconstructs a sample from overlapping diffraction measurements. This technique relies on these overlapping measurements to provide the redundancy information needed for robust reconstruction, but dense scanning increases acquisition time and data volume. Reducing scan overlaps can improve measurement efficiency but weakens the physical constraints available for reconstruction. Diffusion models provide strong learned prior information, offering a promising approach to compensate for the loss of measurement redundancy. However, effectively utilizing this prior information under weak measurement constraints remains a challenge. We propose a progressive coupling posterior sampling (PCPS) method, which allows for a broader exploration of the solution space under high uncertainty while maintaining stronger diffusion-guided consistency as the reconstruction becomes more refined. Experiments on real X-ray ptychographic data at various scan overlaps demonstrate that PCPS consistently improves reconstruction accuracy compared to traditional iterative methods and existing diffusion posterior samplers, especially at low overlaps. These results show that diffusion priors can reduce the measurement redundancy required for accurate ptychography reconstruction.
\end{abstract}


\IEEEpeerreviewmaketitle

\section{Introduction}

\IEEEPARstart{P}{tychography} is a computational imaging technique that reconstructs the complex transmission function of a sample from overlapping diffraction measurements \cite{rodenburg2004phase,faulkner2004movable,thibault2009probe}. The overlap between adjacent scan positions provides measurement redundancy that is essential for robust reconstruction. However, achieving high overlap requires dense scanning with small scan steps, resulting in more diffraction measurements, longer acquisition times, and larger data volumes. Reducing the scan overlap therefore provides a direct way to improve measurement efficiency by covering the same field of view with fewer scan positions. This efficiency comes at the cost of reduced measurement redundancy, which weakens the constraints on the inverse problem and makes accurate reconstruction increasingly challenging. Classical reconstruction methods, including projection methods, the extended Ptychographic Iterative Engine (ePIE), and least-squares maximum-likelihood reconstruction (LSQ-ML), rely primarily on the physical forward model and can degrade substantially in this low-overlap regime \cite{elser2003phase,maiden2009improved,maiden2017further,odstrcil2018iterative}.

Diffusion models provide powerful prior information, offering a promising approach to compensate for the loss of measurement redundancy in low-overlap ptychography. To incorporate this prior information into the inverse problem, diffusion posterior sampling methods combine the learned image distribution with measurement information during inference. Diffusion posterior sampling (DPS) enforces measurement consistency during inverse diffusion and has been extended to ptychography reconstruction \cite{chung2023diffusion,cam2025ptychographic}. A complementary hybrid approach, Ptycho-LDM, uses a conditional latent diffusion model to refine coarse Difference Map reconstructions obtained from reduced probe measurements \cite{saha2025ptycho}. Decoupled annealed posterior sampling (DAPS) enables broader exploration of the solution space by relaxing the dependence on the diffusion trajectory, which is particularly advantageous for complex nonlinear inverse problems \cite{zhang2025improving}. DAPS++ further improves computational efficiency by separating prior-based initialization from measurement-driven refinement \cite{chen2025dapspp}. These advances highlight an important trade-off in diffusion-based reconstruction: more strict adherence to the diffusion trajectory preserves the learned prior structure, while greater deviation from the diffusion trajectory allows for a broader exploration of the solution space in the case of weak or ambiguous measurements.

Motivated by this trade-off, we propose progressive coupling posterior sampling (PCPS), which adapts the balance between measurement-driven exploration and diffusion-guided consistency over the sampling trajectory. At high noise levels, where uncertainty remains large, PCPS maintains greater freedom for measurement-driven exploration. As the noise level decreases and the reconstruction becomes more refined, PCPS progressively strengthens consistency with the diffusion trajectory to better preserve the learned structure. PCPS also naturally supports sparse sampling, enabling fewer sampling steps and lower computational cost. Experiments on real X-ray ptychographic data across different scan overlaps demonstrate that PCPS consistently improves reconstruction accuracy, with particularly clear advantages in the low-overlap regime.

\section{Problem Setup and Background}

\subsection{Ptychography Formula}

Ptychography reconstructs a complex-valued object $x_0\in\mathbb{C}^{m\times n}$ by scanning a localized probe across overlapping regions of the sample and recording the corresponding far-field diffraction intensities. For scan $i$,
\begin{equation}
    I_i = |\mathcal A_ix_0|^2 = |FPM_i x_0|^2,
    \label{eq:ptycho_forward}
\end{equation}
where $M_i$ extracts the local object region, $P$ denotes the probe function, and $F$ is the discrete Fourier transform. The operator $\mathcal A$ denotes the collection of \{$\mathcal A_i$\} over all scan positions.

To account for photon-counting noise in ptychography, a common data-fidelity term is the square-root intensity loss \cite{odstrcil2018iterative}
\begin{equation}
\ell_y(x_0)=
\left\|
\sqrt{y+\epsilon_y}
-
\sqrt{|\mathcal A x_0|^2+\epsilon_y}
\right\|_2^2,
\label{eq:data_loss}
\end{equation}
where $\epsilon_y>0$ is a small constant for numerical stability. 

\subsection{Diffusion Model and DAPS}

In diffusion model-based methods, the forward diffusion process gradually perturbs the clean sample $x_0$ by adding Gaussian noise, with transition \cite{ho2020ddpm}
\begin{equation}
 q(x_t\mid x_{t-1})
 =\mathcal N\left(\sqrt{\alpha_t}x_{t-1},\beta_t I\right),
\label{eq:forward_transition}
\end{equation}
where $\alpha_t=1-\beta_t$. By composing these transitions, $x_t$ can be sampled directly from $x_0$ according to
\begin{equation}
 q_t(x_t\mid x_0)
 =\mathcal N\left(m_tx_0,v_tI\right),
 \ 
 m_t=\sqrt{\bar\alpha_t},
 \ 
 v_t=1-\bar\alpha_t,
\label{eq:forward_marginal}
\end{equation}
with $\bar\alpha_t=\prod_{r=1}^{t}\alpha_r$. 

A noise prediction network $\epsilon_\theta(x_t,t)$, trained to estimate the Gaussian noise added at each diffusion level, can then be used to recover an estimate of the underlying clean sample
\begin{equation}
 \widehat x_0(x_t)
 =\frac{x_t-\sqrt{v_t}\,\epsilon_\theta(x_t,t)}{m_t}.
\label{eq:denoised_estimate}
\end{equation}

Given measurements y, the goal of posterior sampling is to draw clean samples from the posterior distribution $p(x_0\mid y)$. Since direct sampling is generally intractable, diffusion-based methods traverse a sequence of noise-perturbed posterior marginals induced by the forward diffusion process
\begin{equation}
 p_t(x_t\mid y)
 =\int q_t(x_t\mid x_0)p(x_0\mid y)\,dx_0.
\label{eq:noisy_posterior}
\end{equation}
As the noise level approaches zero, this distribution recovers the target clean posterior $p(x_0\mid y)$.

DAPS performs this transition in two steps. To move from the current noise level $t$ to a lower level $s<t$, DAPS first approximates a clean conditional sample
\begin{equation}
x_{0\mid y} \sim p(x_0 \mid x_t,y),
\label{eq:clean_conditional}
\end{equation}
and then re-noises it according to the diffusion marginal
\begin{equation}
x_s \sim q_s(x_s \mid x_{0\mid y})
=
\mathcal N(m_s x_{0\mid y}, v_s I).
\label{eq:daps_transition}
\end{equation}

When $x_t \sim p_t(x_t \mid y)$, this two-step update yields $x_s \sim p_s(x_s \mid y)$. DAPS therefore uses $x_t$ through the clean conditional in Eq. \eqref{eq:clean_conditional}. Once $x_{0\mid y}$ has been obtained, the re-noising step in Eq. \eqref{eq:daps_transition} no longer depends explicitly on $x_t$. This decoupling facilitates exploration at high noise levels, but may hinder the preservation of established structure as the noise level decreases.

\begin{figure}[ht]
\centering
\includegraphics[width=\linewidth]{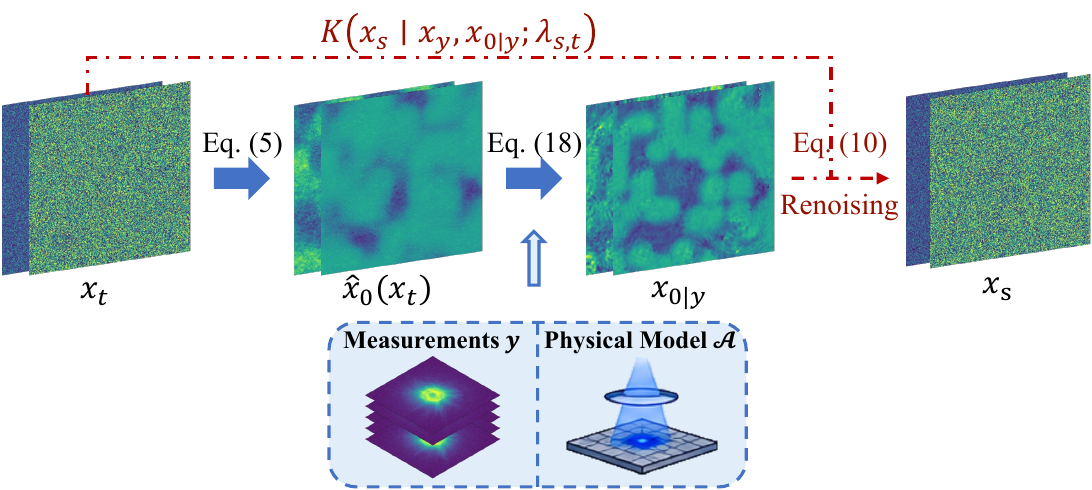}
\caption{Overview of PCPS. For a transition from noise level $t$ to $s<t$, the denoiser produces $\hat{x}_0(x_t)$, measurement refinement with $y$ and $\mathcal{A}$ yields $x_{0|y}$, and the Gaussian coupling combines the refined clean estimate, the residual of the current noisy state, and fresh Gaussian noise to sample $x_s$. The predetermined coefficient $\lambda_{s,t}$ increases toward the clean endpoint.}
\label{fig:apr_schematic}
\end{figure}

\section{Progressive Coupling for Posterior Sampling}

This section introduces the proposed progressive coupling posterior sampling (PCPS) method, which balances measurement-driven exploration and diffusion-guided consistency over the sampling trajectory. Figure~\ref{fig:apr_schematic} provides an overview of the sampling procedure.


\subsection{Posterior Marginal-Preserving Gaussian Coupling}

DAPS maps $p_t(x_t\mid y)$ to $p_s(x_s\mid y)$ through independent re-noising where $x_s$ is decoupled with $x_t$. However, preservation of the target marginal $p_s(x_s\mid y)$ does not impose a fully decoupled 
transition. PCPS therefore constructs a family of Gaussian transition kernels with a coupling coefficient that controls how much information from $x_t$ is retained while maintaining the desired posterior marginal.

\begin{proposition}[Gaussian coupling]
For a predetermined coupling coefficient
$\lambda_{s,t}\in[0,1]$, $s<t$, define the Gaussian transition kernel
\begin{equation}
K(x_s\mid x_t,x_0;\lambda_{s,t})
=
\mathcal N\!\left(
\mu_{s,t}^{\lambda},
\sigma_{s,t}^2 I
\right),
\label{eq:coupling}
\end{equation}
where
\begin{equation}
\begin{aligned}
\mu_{s,t}^{\lambda}
&=
m_sx_0
+\lambda_{s,t}b_{s,t}(x_t-m_tx_0),\\
\sigma_{s,t}^2
&=
v_s-\lambda_{s,t}^2b_{s,t}^2v_t.
\end{aligned}
\end{equation}
Here, $
b_{s,t}
=
\sqrt{\frac{\bar\alpha_t}{\bar\alpha_s}}
\frac{v_s}{v_t}
$. If
$x_t\mid x_0\sim q_t(x_t\mid x_0)$,
then the induced marginal satisfies
\begin{equation}
\int
K(x_s\mid x_t,x_0;\lambda_{s,t})
q_t(x_t\mid x_0)\,dx_t
=
q_s(x_s\mid x_0).
\label{eq:coupling_marginal}
\end{equation}
\end{proposition}

\begin{IEEEproof}
Since $s<t$ implies $\bar\alpha_t\leq\bar\alpha_s$,
$b_{s,t}^2v_t\leq v_s$, so $\sigma_{s,t}^2\geq0$ for
$\lambda_{s,t}\in[0,1]$.
With $x_t\mid x_0\sim\mathcal N(m_tx_0,v_tI)$, the transition
induces a Gaussian $x_s\mid x_0$ with mean $m_sx_0$ and covariance
\[
\lambda_{s,t}^2b_{s,t}^2v_tI+\sigma_{s,t}^2I=v_sI.
\]
Hence the induced marginal is $q_s(x_s\mid x_0)$.
\end{IEEEproof}

At $\lambda_{s,t}=0$, the transition kernel reduces to $q_s(x_s\mid x_0)$, corresponding to independent re-noising. At $\lambda_{s,t}=1$, it recovers $q(x_s\mid x_t,x_0)$, the conditional transition induced by the original forward diffusion process. Intermediate values provide continuously varying dependence on $x_t$. This construction is related to DDIM-type non-Markovian transitions that alter cross-level dependence while preserving diffusion marginals \cite{song2021ddim}.

\begin{proposition}[posterior marginal preservation]
Let $x_t\sim p_t(x_t\mid y)$ and draw $x_{0\mid y}\sim p(x_0\mid x_t,y)$. If $\lambda_{s,t}$ is fixed independently of the realized $x_t$ and $x_{0\mid y}$, then drawing $x_s$ from Eq.~\eqref{eq:coupling} with $x_0=x_{0\mid y}$ gives
\begin{equation}
x_s\sim p_s(x_s\mid y).
\label{eq:posterior_preservation}
\end{equation}
\end{proposition}

\begin{IEEEproof}
The joint density of the current noisy state and the sampled clean variable satisfies
\begin{align}
p_t(x_t\mid y)p(x_0\mid x_t,y)
&=p(x_0,x_t\mid y) \nonumber\\
&=p(x_0\mid y)p(x_t\mid x_0,y) \nonumber\\
&=p(x_0\mid y)q_t(x_t\mid x_0),
\label{eq:posterior_joint_identity}
\end{align}
where the last equality follows from the forward diffusion model, under which $x_t$ is conditionally independent of $y$ given $x_0$.

The output density induced by the two-step sampling procedure is therefore
\begin{equation}
    \small
    \begin{aligned}
    \tilde{p}(x_s\mid y)
    &=\iint K(x_s\mid x_t,x_0;\lambda_{s,t})\,p(x_0\mid x_t,y)\,p_t(x_t\mid y)\,dx_t\,dx_0 \nonumber\\
    &=\int p(x_0\mid y)\left[\int K(x_s\mid x_t,x_0;\lambda_{s,t})\,q_t(x_t\mid x_0)\,dx_t\right]dx_0.
    \end{aligned}
    \label{eq:posterior_output_density}
\end{equation}

Since $\lambda_{s,t}$ is independent of the realized $x_t$ and $x_0$, Proposition~1 applies conditionally for every fixed $x_0$ and gives
\begin{equation}
\int K(x_s\mid x_t,x_0;\lambda_{s,t})\,q_t(x_t\mid x_0)\,dx_t
=
q_s(x_s\mid x_0).
\label{eq:conditional_marginal_mapping}
\end{equation}
Substituting Eq. \eqref{eq:conditional_marginal_mapping} into Eq. \eqref{eq:posterior_output_density} yields
\begin{align}
\tilde{p}(x_s\mid y)
&=\int q_s(x_s\mid x_0)p(x_0\mid y)\,dx_0 \nonumber\\
&=p_s(x_s\mid y),
\end{align}
which proves Eq. \eqref{eq:posterior_preservation}.
\end{IEEEproof}

A predetermined coupling schedule can therefore vary across noise levels while preserving the noisy posterior marginal at every visited level under exact clean conditional sampling. This permits weak coupling at high noise levels to allow greater exploration and progressively stronger coupling at low noise levels to preserve the structure already established in the current state.

\subsection{Progressive Coupling Sampler}

The exact clean conditional sample in Eq.~\eqref{eq:clean_conditional} is infeasible and is approximated using Langevin dynamics with measurement consistency and a diffusion-based reference term~\cite{zhang2025improving}. Starting from $z^{(0)}=\widehat x_0(x_t)$, $K_t$ refinement steps are performed as
\begin{equation}
\begin{aligned}
 z^{(k+1)}={}&z^{(k)}
 -\eta_t\nabla_z\ell_y(z^{(k)})
 -\eta_t\kappa_t\left(z^{(k)}-\widehat x_0(x_t)\right)\\
 &+\gamma_t\sqrt{2\eta_t}\,\xi^{(k)},
 \qquad \xi^{(k)}\sim\mathcal N(0,I).
\end{aligned}
\label{eq:refinement}
\end{equation}
where $\eta_t$ is the step size, $\kappa_t$ controls the reference term, and $\gamma_t$ is the stochastic update scale. The measurement-loss gradient enforces consistency with the ptychographic data, while the reference term keeps the update close to the clean estimate supplied by the diffusion model. After $K_t$ steps, $x_{0\mid y}$ is set to $z^{(K_t)}$.

To realize progressive coupling from weak to strong, PCPS selects $M$ positive noise levels along the diffusion timeline as the visited sampling states, ordered from high to low noise as
\begin{equation}
    \mathcal T_M=\{t_0>t_1>\cdots>t_{M-1}>0\}.
\end{equation}
Here, $M$ denotes the number of visited noise levels. The adjacent states define $M-1$ transitions, with the $j$-th transition given by $t_j\rightarrow t_{j+1}$, where $j=0,\ldots,M-2$.

To determine the coupling strength according to the sampling progress, the noise level is represented by the log-SNR
\begin{equation}
    \ell_t=\log\frac{\bar\alpha_t}{1-\bar\alpha_t}.
\end{equation}

The normalized progress of the $j$-th transition is defined as
\begin{equation}
    r_j=
\frac{\ell_{t_j}-\ell_{t_0}}
{\ell_{t_{M-2}}-\ell_{t_0}},
\qquad j=0,\ldots,M-2,
\end{equation}
so that the first transition satisfies $r_0=0$ and the last transition satisfies $r_{M-2}=1$. The coupling coefficient is then assigned using a cosine schedule,
\begin{equation}
    \lambda_j=\frac{1-\cos(\pi r_j)}{2},
    \label{eq:cosine_schedule}
\end{equation}
such that the coupling strength progressively increases as sampling proceeds from high to low noise levels.

For each transition $t_j\rightarrow t_{j+1}$, a clean estimate is first obtained from the current state $x_{t_j}$ and then refined using the measurement information to obtain $x_{0\mid y}$. The next noisy state $x_{t_{j+1}}$ is subsequently sampled using the Gaussian coupling with $x_0=x_{0\mid y}$ and $\lambda_{t_{j+1},t_j}=\lambda_j$. After reaching the final noise level $t_{M-1}$, one final denoising and refinement step is performed, and the resulting $x_{0\mid y}$ is returned as the reconstruction.

\subsection{Sparse Log-SNR Sampling}

Proposition 1 applies to arbitrary pairs $s<t$, allowing transitions between nonadjacent diffusion timesteps. For PCPS-M, the $M$ visited levels are selected by uniformly spacing their target log-SNR values between the initial noise level and the final positive noise level of the full sampling trajectory, with each target mapped to the nearest training timestep. The two endpoints are retained for all $M$. PCPS-1000 uses the full 1000-step schedule.
\section{Experiments}

\subsection{Data, Evaluation, and Baselines}

We use real X-ray ptychographic data from ptychoNN \cite{babu2023deep}, split at the full-resolution object level into 67 training, 16 validation, and 11 test objects. The diffusion prior is trained only on the training set using $768\times768$ complex crops with real and imaginary channels. Subsampling the full scan positions produces 60\%, 50\%, 40\%, 30\%, and 20\% overlap.

For each object, 2000 iterations of 3-mode LSQ-ML in tike \cite{doecode_28526} reconstruct the reference object and probe from complete scans. All experiments use the corresponding known reference probe.

Each reconstruction is formed by stitching four overlapping $768\times768$ windows. Neighboring windows are phase-aligned over their common support and averaged, followed by global phase alignment of the stitched reconstruction $\widehat x$ to the reference. Reconstruction accuracy is evaluated using
\begin{equation}
\operatorname{NRMSE}(\widehat x,x_{\rm ref})
=
\min_{\phi}
\frac{\|e^{\mathrm i\phi}\widehat x-x_{\rm ref}\|_2}
{\|x_{\rm ref}\|_2}.
\label{eq:nrmse}
\end{equation}

Comparisons are conducted against LSQ-ML \cite{odstrcil2018iterative}, DPS \cite{chung2023diffusion}, DAPS \cite{zhang2025improving}, and DAPS++ \cite{chen2025dapspp}. LSQ-ML is run for 2000 iterations at each overlap. The same diffusion prior, probe, data loss, reconstruction windows, and evaluation procedure are used for all diffusion-based methods. DAPS-M and PCPS-M additionally share the same noise grid and refinement parameters, differing only in the coupling step.

\subsection{Reconstruction Accuracy}

\begin{table}[t]
\centering
\caption{stitched complex NRMSE at different scan overlaps. reported as mean ± standard deviation over the test objects.}
\label{tab:overlap}
\footnotesize
\setlength{\tabcolsep}{1.5pt}
\renewcommand{\arraystretch}{0.95}
\begin{tabular}{@{}lccccc@{}}
\toprule
Method & 60\% & 50\% & 40\% & 30\% & 20\% \\
\midrule
DPS    & 0.390±0.335 & 0.347±0.195 & 0.374±0.107
       & 0.408±0.229 & 0.446±0.079 \\
LSQ-ML & 0.182±0.030 & 0.256±0.042 & 0.304±0.051
       & 0.403±0.064 & 0.508±0.077 \\
\midrule
\multicolumn{6}{@{}l}{\textit{100 steps}}\\
DAPS   & 0.113±0.004 & 0.133±0.005 & 0.145±0.006
       & 0.165±0.006 & 0.181±0.006 \\
DAPS++ & 0.125±0.009 & 0.143±0.015 & 0.158±0.016
       & 0.205±0.022 & 0.262±0.047 \\
PCPS   & \underline{0.106±0.007} & \underline{0.118±0.008} & \underline{0.125±0.008}
       & \underline{0.140±0.009} & \underline{0.151±0.008} \\
\midrule
\multicolumn{6}{@{}l}{\textit{500 steps}}\\
DAPS   & 0.117±0.005 & 0.138±0.008 & 0.149±0.009
       & 0.168±0.008 & 0.184±0.008 \\
DAPS++ & 0.113±0.012 & 0.133±0.009 & 0.143±0.010
       & 0.166±0.017 & 0.182±0.024 \\
PCPS   & \textbf{0.097±0.005} & \textbf{0.109±0.006} & \textbf{0.116±0.006}
       & \textbf{0.129±0.008} & \textbf{0.135±0.007} \\
\bottomrule
\label{tab:main_results}
\end{tabular}
\end{table}

\begin{figure}[t]
\centering
\includegraphics[width=\linewidth]{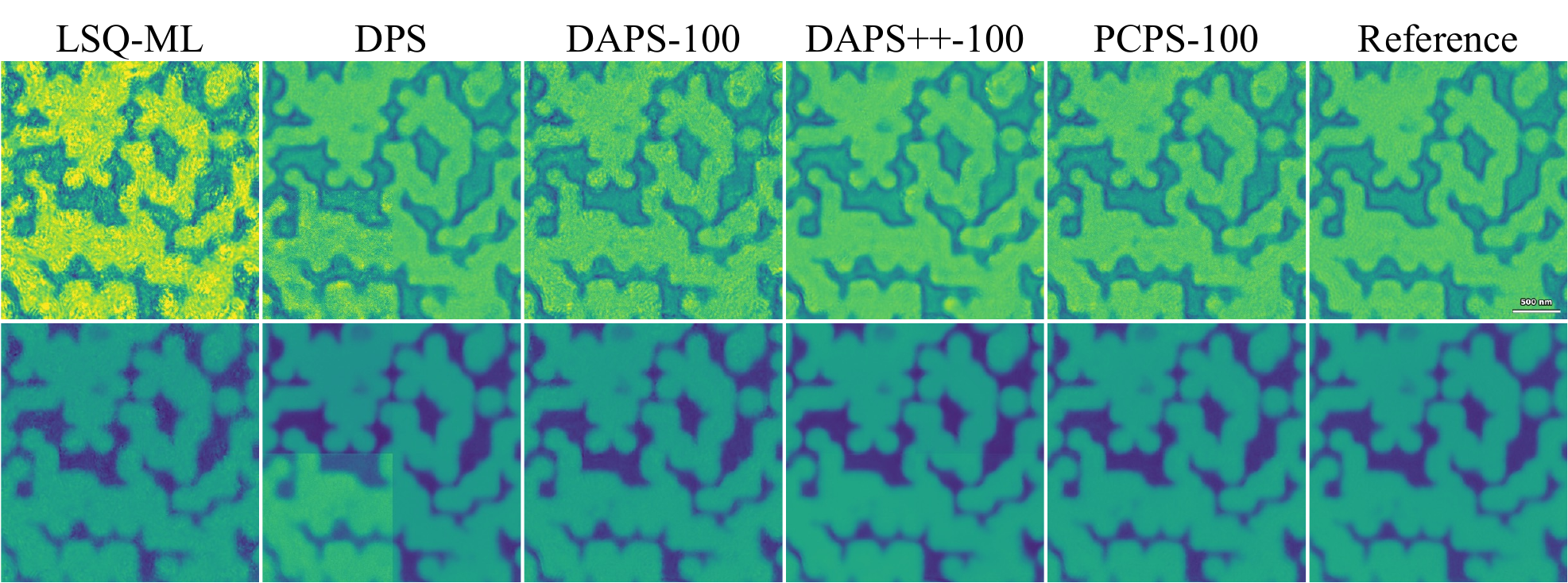}
\caption{Qualitative reconstructions at 40\% overlap. The top and bottom rows show the amplitude and phase, respectively. The scale bar denotes 500 nm.}
\label{fig:qualitative_results}
\end{figure}

Table~\ref{tab:main_results} reports the stitched complex NRMSE at five scan overlaps. PCPS consistently achieves the lowest NRMSE under both the 100- and 500-step settings, outperforming matched DAPS and DAPS++ at every overlap. The advantage over DAPS becomes more pronounced as the overlap decreases, indicating that progressive coupling is particularly beneficial as the phase retrieval problem becomes more challenging. Increasing the sampling budget from 100 to 500 steps further improves PCPS across all overlap levels, whereas DAPS does not benefit from the additional steps. DAPS++ also improves with a larger sampling budget but remains less accurate than PCPS. In comparison, LSQ-ML degrades substantially at low overlaps, while DPS exhibits both higher reconstruction errors and larger variations. Figure~\ref{fig:qualitative_results} shows qualitative results at 40\% overlap, consistent with the quantitative comparisons in Table~\ref{tab:main_results}.

\subsection{Sparse Sampling and Computational Cost}

\begin{figure}[t]
\centering
\includegraphics[width=0.9\linewidth]{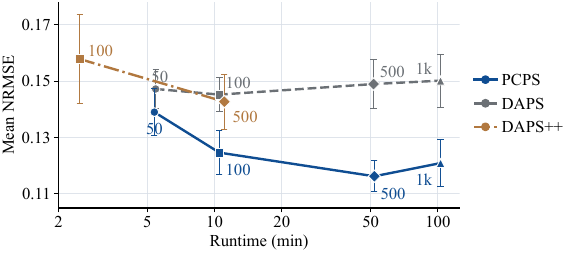}
\caption{Accuracy--runtime tradeoff at 40\% overlap for different sampling budgets. Labels indicate the number of sampling steps. Lower-left points provide better tradeoffs.}
\label{fig:accuracy_cost}
\end{figure}

Figure~\ref{fig:accuracy_cost} compares reconstruction accuracy and averaged runtime at 40\% overlap, with runtime measured on an NVIDIA GeForce RTX 3090 GPU. PCPS consistently provides a better accuracy--cost tradeoff than matched DAPS across the tested sampling budgets. PCPS-500 achieves the lowest mean NRMSE, while PCPS-100 already attains strong accuracy at substantially lower cost and outperforms DAPS++-500 at a similar runtime. These results show that PCPS remains effective under sparse sampling.

\subsection{Coupling Schedule Analysis}

\begin{figure}[t]
\centering
\includegraphics[width=0.95\linewidth]{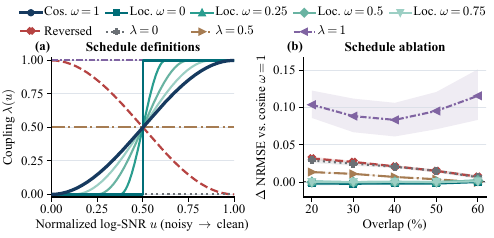}
\caption{Coupling schedule analysis. (a) Compared coupling schedules. (b) NRMSE difference relative to the full-width cosine schedule across scan overlaps.}
\label{fig:schedule_ablation}
\end{figure}

PCPS-100 is used to compare the full-width cosine schedule in Eq.(\ref{eq:cosine_schedule}) with fixed couplings $\lambda=0$, $0.5$, and $1$, a reversed cosine, and local cosine schedules centered at $r=1/2$ with different transition widths, as shown in Fig.~\ref{fig:schedule_ablation}(a). Here, $w=0$ denotes a hard switch and $w=1$ recovers the full-width cosine.

Figure~\ref{fig:schedule_ablation}(b) reports the NRMSE difference relative to the full-width cosine. The increasing schedules remain close to the baseline across all overlaps, indicating limited sensitivity to the transition width. In contrast, fixed and reversed schedules generally perform worse, with $\lambda=1$ showing the largest degradation. These results support progressively increasing the coupling as the noise level decreases.

\section{Conclusion}

This paper presents PCPS for low-overlap ptychographic reconstruction. Experiments on real X-ray data demonstrate improved reconstruction accuracy over iterative and diffusion-based baselines, particularly at low overlaps, while sparse sampling reduces computational cost with limited accuracy loss. These results indicate that diffusion priors can enable more measurement-efficient ptychographic imaging.



\bibliography{Templates/ref}

\end{document}